\documentclass[]{article}
\usepackage{proceed2e}

\usepackage{times}

\usepackage{amsmath}
\usepackage{amsfonts}
\usepackage{amsthm}

\usepackage{graphicx}
\usepackage{algorithm}
\usepackage{algorithmic}
\usepackage{color}
\usepackage{natbib}

\DeclareMathOperator*{\argmin}{arg\,min}
\DeclareMathOperator*{\argmax}{arg\,max}

\newcommand{\bx}{\mathbf{x}}
\newcommand{\hb}{H_{\mathcal{B}}}

\newcommand{\by}{\mathbf{y}}
\newcommand{\bmu}{\boldsymbol{\mu}}
\newcommand{\btheta}{\boldsymbol{{\boldsymbol{\theta}}}}

\newcommand{\dt}[2]{\left\langle #1,#2 \right\rangle}
\newcommand*{\defeq}{\mathrel{\vcenter{\baselineskip0.5ex \lineskiplimit0pt
                     \hbox{\scriptsize.}\hbox{\scriptsize.}}}%
                     =}
\newcommand{\Ex}{\mathbb{E}}

\newtheorem{proposition}{Proposition}
\newtheorem{definition}{Definition}
\newtheorem{corollary}{Corollary}

\newcommand{\comment}[1]{}

\title{Bethe Projections for Non-Local Inference}

\author{ {\bf Luke Vilnis\textsuperscript{*}} \\
UMass Amherst\\
luke@cs.umass.edu\\
\And
{\bf David Belanger\textsuperscript{*}}  \\
UMass Amherst          \\
belanger@cs.umass.edu\\
\And
{\bf Daniel Sheldon}  \\
UMass Amherst          \\
sheldon@cs.umass.edu\\
\And
{\bf Andrew McCallum}  \\
UMass Amherst          \\
mccallum@cs.umass.edu\\
}

\begin{document}

\maketitle

\let\thefootnote\relax\footnote{\textsuperscript{*} Equal contribution.}

\begin{abstract}
% way too long, remove question marks
Many inference problems in structured prediction are naturally solved by augmenting a tractable dependency structure with complex, non-local auxiliary objectives. This includes the mean field family of variational inference algorithms, soft- or hard-constrained inference using Lagrangian relaxation or linear programming, collective graphical models, and forms of semi-supervised learning such as posterior regularization. We present a method to discriminatively \emph{learn} broad families of inference objectives, capturing powerful non-local statistics of the latent variables, while maintaining tractable and provably fast inference using non-Euclidean projected gradient descent with a distance-generating function given by the Bethe entropy. We demonstrate the performance and flexibility of our method by (1) extracting structured citations from research papers by learning soft global constraints, (2) achieving state-of-the-art results on a widely-used handwriting recognition task using a novel learned non-convex inference procedure, and (3) providing a fast and highly scalable algorithm for the challenging problem of inference in a collective graphical model applied to bird migration.
\end{abstract}

\section{INTRODUCTION}
\label{introduction}

Structured prediction has shown great success in modeling problems with complex dependencies between output variables. Practitioners often use undirected graphical models, which encode conditional dependency relationships via a graph. However, the tractability of exact inference in these models is limited by the graph's \emph{treewidth}, often yielding a harsh tradeoff between model expressivity and tractability.

Graphical models are good at modeling local dependencies between variables, such as the importance of surrounding context in determining the meaning of words or phrases. However, their sensitivity to cyclic dependencies often renders them unsuitable for modeling preferences for certain globally consistent states. For example, in the canonical NLP task of part-of-speech tagging, there is no clear way to enforce the constraint that every sentence have at least one verb without increasing the likelihood that \emph{every} token is predicted to be a verb.

Concretely, exact marginal inference in a discrete graphical model can be posed as the following optimization problem
\begin{align}
\bmu^* =  \argmin_{\bmu \in \mathcal{M}} & \;  -H(\bmu) -  \dt{\btheta}{\bmu}, \label{eq:kl2}
\end{align}
where $\bmu$ is a concatenated vector of node and clique marginals, $H(\bmu)$ is the entropy, $\mathcal{M}$ is the marginal polytope, and ${\boldsymbol{\theta}}$ are parameters. Here we face a tradeoff: adding long-range dependencies directly to the model increases the clique size and thus the complexity of the problem and size of $\bmu$, rendering inference intractable. However, the linear scoring function $\btheta$ breaks down over cliques, preventing us from enforcing global regularities in any other way. In this work, we propose to augment the inference objective \eqref{eq:kl2} and instead optimize 
\begin{align}
  \bmu^* = \argmin_{\bmu \in \mathcal{M}} -H(\bmu) - \dt{\btheta}{\bmu} +  L_{\boldsymbol{\psi}}(\bmu).  \label{eq:aug-inf2}
\end{align}
Here, $L_{\boldsymbol{\psi}}$ is some arbitrary parametric function of the entire concatenated marginal vector, where ${\boldsymbol{\psi}}$ may depend on input features. Since $L_{\boldsymbol{\psi}}$ is non-linear, it can enforce many types of non-local properties. Interestingly, whenever $L_{\boldsymbol{\psi}}$ is convex, and whenever inference is easy in the underlying model, i.e., solving~\eqref{eq:kl2} is tractable, we can solve~\eqref{eq:aug-inf2} using non-Euclidean projected gradient methods using the Bethe entropy as a distance-generating function.  Unlike many message-passing algorithms, our procedure maintains primal feasibility across iterations, allowing its use as an ~\emph{anytime} algorithm. Furthermore, for non-convex $L_{\boldsymbol{\psi}}$, we also show convergence to a local optimum of~\eqref{eq:aug-inf2}. Finally, we present algorithms for discriminative learning of the parameters ${\boldsymbol{\psi}}$. In a slight abuse of terminology, we call $L_{\boldsymbol{\psi}}$ a \emph{non-local energy function}.

Ours is not the first work to consider modeling global preferences by augmenting a tractable base inference objective with non-local terms. For example, generalized mean-field variational inference algorithms augment a tractable distribution (the $Q$ distribution) with a non-linear, non-convex global energy function that scores terms in the full model (the $P$ distribution) using products of marginals of $Q$~\citep{wainwright2008graphical}. This is one special case of our non-local inference framework, and we present algorithms for solving the problem for much more general $L_{\boldsymbol{\psi}}$, with compelling applications.

Additionally, the modeling utility provided by global preferences has motivated work in \emph{dual decomposition}, where inference in loopy or globally-constrained models is decomposed into repeated calls to inference in tractable independent subproblems~\citep{komodakis2007mrf,sontag2011introduction}. It has seen wide success due to its ease of implementation, since it reuses existing inference routines as black boxes. However, the technique is restricted to modeling linear constraints, imposed \emph{a priori}. Similarly, these types of constraints have also been imposed on expectations of the posterior distribution for use in semi-supervised learning, as in~\emph{posterior regularization} and \emph{generalized expectation}~\citep{ganchev2010posterior,mann2010generalized}. In contrast, our methods are designed to discriminatively learn expressive inference procedures, with minimal domain knowledge required, rather than regularizing inference and learning.

First, we provide efficient algorithms for solving the marginal inference problem \eqref{eq:aug-inf2} and performing MAP prediction in the associated distribution, for both convex and non-convex global energy functions.  After that, we provide a learning algorithm for $\btheta$ and the parametrized $L_{\boldsymbol{\psi}}$ functions using an interpretation of~\eqref{eq:aug-inf2} as approximate variational inference in a a probabilistic model. All of our algorithms are easy to implement and rely on simple wrappers around black-box inference subroutines.  

Our experiments demonstrate the power and generality of our approach by achieving state-of-the-art results on several tasks. We extract accurate citations from research papers by learning discriminative global regularities of valid outputs, outperforming a strong dual decomposition-based baseline~\citep{anzaroot2014learning}. In a benchmark OCR task~\citep{koller2004max}, we achieve state-of-the-art results with a learned non-convex, non-local energy function, that guides output decodings to lie near dictionary words. Finally, our general algorithm for solving \eqref{eq:aug-inf2} provides large speed improvements for the challenging task of inference in chain-structured \emph{collective graphical models} (CGMs), applied to bird migration~\citep{sheldon2011collective}. 

\section{BACKGROUND}

Let $\by = (y_1,\ldots, y_n)$ denote a set of discrete variables and $\bx$ be a collection of input features. We define the conditional distribution $P_{\boldsymbol{\theta}}(\by|\bx) = \exp(\dt{{\boldsymbol{\theta}}(\bx)}{S(\by)})/Z$, where $S(\by)$ is a mapping from $\by$ to a set of sufficient statistics, ${\boldsymbol{\theta}}(\bx)$ is a differentiable vector-valued mapping,  and $Z = \sum_\by \exp(\dt{{\boldsymbol{\theta}}}{S(\by)})$. Conditional random fields (CRFs) assume that $(y_1,\ldots, y_n)$ are given a graph structure and $S(\by)$ maps $\by$ to a 0-1 vector capturing joint settings of each clique \citep{lafferty2001conditional}. Going forward, we often suppress the explicit dependency of ${\boldsymbol{\theta}}$ on $\bx$. For fixed ${\boldsymbol{\theta}}$, the model is called a Markov random field (MRF). 
 
Given a distribution $P(\by)$, define the expected sufficient statistics operator $\mu(P) = \Ex_{P}[S(\by)]$.  For the CRF  statistics $S(\by)$ above, $\bmu$ is a concatenated vector of node and clique marginals. Therefore, \textit{marginal inference}, the task of finding the marginal distribution of $P_{\boldsymbol{\theta}}(\by|\bx)$ over $\by$, is equivalent to computing the expectation $\mu(P_{\boldsymbol{\theta}}(\by|\bx))$.

For tree-structured graphical models, $P_{\boldsymbol{\theta}}(\by | \bx) \longleftrightarrow \bmu(P_{\boldsymbol{\theta}}(\by | \bx))$ is a bijection, though this is not true for general graphs. Furthermore, for trees the entropy  $H(P_{\boldsymbol{\theta}}(\by | \bx))$ is equal to the Bethe entropy $\hb\left(\bmu(P_{\boldsymbol{\theta}}(\by | \bx))\right)$, defined, for example, in~\citet{wainwright2008graphical}. The \textit{marginal polytope} $\mathcal{M}$ is the set of $\bmu$ that correspond to some $P_{\boldsymbol{\theta}}$.

As mentioned in the introduction, marginal inference can be posed as the optimization problem \eqref{eq:kl2}. MAP inference finds the joint setting $\by$ with maximum probability. For CRFs, this is equivalent to
\begin{align}
\argmin_\by \dt{-{\boldsymbol{\theta}}(\bx)}{S(\by)}. \label{eq:MAP}
\end{align}
For tree-structured CRFs, marginal and MAP inference can be performed efficiently using dynamic programming. Our experiments focus on such graphs. However, the inference algorithms we present can be extended to general graphs wherever marginal inference is tractable using a convex entropy approximation and a local polytope relaxation.

\section{MARGINAL INFERENCE WITH NON-LOCAL ENERGIES}
We move beyond the standard inference objective~\eqref{eq:kl2}, augmenting it with a non-local energy term as in \eqref{eq:aug-inf2}:
\begin{align*}
  \bmu^* = \argmin_{\bmu \in \mathcal{M}} -\hb(\bmu) - \dt{{\boldsymbol{\theta}}}{\bmu }+  L_{\boldsymbol{\psi}}(\bmu).
\end{align*}

Here, $L_{\boldsymbol{\psi}}$ is some arbitrary parametrized function of the marginals, and ${\boldsymbol{\psi}}$ may depend on input features $\bx$.

Intuitively, we are augmenting the inference objective \eqref{eq:kl2} by allowing it to optimize a broader set of tradeoffs -- not only between expected node scores, clique scores, and entropy, but also global functions of the marginals. To be concrete, in our citation extraction experiments (Section~\ref{citation-experiments}), for example, we employ the simple structure:
\begin{align}
L_{\boldsymbol{\psi}}(\bmu) = \sum_j {\boldsymbol{\psi}}_j \ell_j(\bmu) \label{eq:ell1},
\end{align}
Where each $\ell_j$ is a univariate convex function and each ${\boldsymbol{\psi}}_j$ is constrained to be non-negative, in order to maintain the overall convexity of $L_{\boldsymbol{\psi}}$. We further employ
\begin{align}
\ell_j(\bmu) = \tilde{\ell}_j\left(a_j^\top\bmu\right),\label{eq:ell2}
\end{align}
where $a_j$ encodes a `linear measurement' of the marginals and $\tilde{\ell}_j$ is some univariate convex function. 

\section{VARIATIONAL INTERPRETATION AND MAP PREDICTION}
\label{sec:var-and-map}

We next provide two complementary interpretations of~\eqref{eq:aug-inf2} as variational inference in a class of tractable probability distributions over $\by$. They yield precisely the same variational expression. However, both are useful because the first helps motivate a MAP prediction algorithm, while the second helps characterize our learning algorithm in Section~\ref{sec:learning} as (approximate) variational EM. 

\begin{proposition}

\label{prop:var}
For fixed ${\boldsymbol{\theta}}$ and $L_{\boldsymbol{\psi}}$, the output $\bmu^*$ of inference in the augmented objective~\eqref{eq:aug-inf2} is equivalent to the output of standard inference \eqref{eq:kl2}  in an MRF with the same clique structure as our base model, but with a modified parameter $\tilde{{\boldsymbol{\theta}}} = {\boldsymbol{\theta}} -\nabla L_{\boldsymbol{\psi}}(\bmu^*)$ . 
\end{proposition}
\begin{proof}

\renewcommand{\qedsymbol}{}
Forming a Lagrangian for~\eqref{eq:aug-inf2}, the stationarity conditions with respect to the variable $\bmu$ are:
\begin{align}
  0 &= -({\boldsymbol{\theta}} - \nabla L_{\boldsymbol{\psi}}(\bmu^*)) -\nabla\hb(\bmu^*) + \nabla_{\bmu}C(\bmu,\boldsymbol{\lambda}),  \label{eq:aug-inf-opt}
\end{align}
where $C(\bmu,\boldsymbol{\lambda})$ are collected terms relating to the marginal polytope constraints. The proposition follows because~\eqref{eq:aug-inf-opt} is the same as the stationarity conditions for
\begin{align}
  \bmu^* = \argmin_{\bmu \in \mathcal{M}} -\dt{{\boldsymbol{\theta}} -\nabla L_{\boldsymbol{\psi}}(\bmu^*)}{\bmu} -\hb(\bmu).  \; \; \square \label{eq:aug-inf-grad-param}
\end{align}
\vspace{-15pt}
\end{proof}

Therefore, we can characterize a joint distribution over $\by$ by first finding $\bmu^*$ by solving~\eqref{eq:aug-inf2} and then defining an MRF over $\by$ with parameters $\tilde{{\boldsymbol{\theta}}}$. Even more conveniently, our inference technique in Section~\ref{sec:optimization} iteratively estimates $\tilde{{\boldsymbol{\theta}}}$ on the fly, namely via the dual iterate ${\boldsymbol{\theta}}_t$ in Algorithm~\ref{alg:b-rda}.%, and thus the gradient of $L_{\boldsymbol{\psi}}$ at the optimal marginals need never be actually computed as a separate step. 

Ultimately, in many prediction problems we seek a single output configuration $\by$ rather than an inferred distribution over outputs. Proposition~\ref{prop:var} suggests a simple prediction procedure: first, find the variational distribution over $\by$ parametrized as an MRF with parameter $\tilde{{\boldsymbol{\theta}}}$. Then, perform MAP in this MRF. Assuming an available marginal inference routine for this MRF, we assume the tractability of MAP -- for example using a dynamic program. We avoid predicting $\by$ by locally maximizing nodes' marginals, since this would not necessarily yield feasible outputs. 

Instead of solving~\eqref{eq:aug-inf2}, we could have introduced global energy terms to the MAP objective~\eqref{eq:MAP} that act directly on values $S(\by)$ rather than on expectations $\bmu$, as in~\eqref{eq:aug-inf2}. However, this yields a difficult combinatorial optimization problem for prediction and does not yield a natural way to learn the parametrization of the global energy. Section~\ref{citation-experiments} demonstrates that using energy terms defined on marginals, and performing MAP inference in the associated MRF, performs as well or better than an LP technique designed to directly perform MAP subject to global penalty terms.

Our second variational interpretation characterizes $\bmu^*$ as a variational approximation to a complex joint distribution: 
\begin{align}
P_c(\by | \bx) = (1/Z_{{\boldsymbol{\theta}},{\boldsymbol{\psi}}})P_{\boldsymbol{\theta}}(\by | \bx)P_{\boldsymbol{\psi}}(\by | \bx).  \label{eq:pc}
\end{align}

We assume that isolated marginal inference in $P_{\boldsymbol{\theta}}(\by | \bx)$ is tractable, while $P_{\boldsymbol{\psi}}(\by | \bx)$ is an alternative structured distribution over $\by$ for which we do not have an efficient inference algorithm.  Specifically, we assume that $\eqref{eq:kl2}$ can be solved for $P_{\boldsymbol{\theta}}$. Furthemore, we assume that $P_{\boldsymbol{\psi}}(\by | \bx) \propto \exp\left(L_{\boldsymbol{\psi}}(S(\by); \bx)\right)$, where $L_{\boldsymbol{\psi}}(\cdot; \bx)$ is a convex function, conditional on input features $\bx$.  Going forward, we will often surpress the dependence of $L_{\boldsymbol{\psi}}$ on $\bx$.  Above, $Z_{{\boldsymbol{\theta}},{\boldsymbol{\psi}}}$ is the normalizing constant of the combined distribution.  Note that if $L$ was linear, inference in both $P_{\boldsymbol{\psi}}(\by | \bx)$ and $P_c(\by | \bx)$ would be tractable, since the distribution would decompose over the same cliques as $P_{\boldsymbol{\theta}}(\by | \bx)$. 

Not surprisingly,~\eqref{eq:pc} is intractable to reason about, due to the non-local terms in~\eqref{eq:aug-inf2}, so we approximate it with a variational distribution $Q(\by)$. The connection between this variational approximation and Proposition~\ref{prop:var} is derived in Appendix~\ref{sec:var-inference}. Here, we assume no clique structure on $Q(\by)$, but show that minimizing a variational approximation of $KL\left(Q(\by) || P_c(\by | \bx\right))$, for a given $\bx$, yields a $Q$ that is parametrized compactly as the MRF in Proposition~\ref{prop:var}. We discuss the relationship between this and general mean-field inference in Section~\ref{sec:rw}.

Although the analysis of this section assumes convexity of $L_{\boldsymbol{\psi}}$, our inference techniques can be applied to non-convex $L_{\boldsymbol{\psi}}$, as discussed in Section~\ref{sec:non-convex}, and our learning algorithm produces state-of-the-art results even in the non-convex regime for a benchmark OCR task. 

\section{RELATED MODELING TECHNIQUES}
\label{sec:rw}
\textbf{Mean field} variational inference in undirected graphical models is a particular application of our inference framework, with a non-convex $L_{\boldsymbol{\psi}}$~\citep{wainwright2008graphical}. 
The technique estimates marginal properties of a complex joint distribution $P$ using the clique marginals $\bmu$ of some tractable base distribution $Q$, not necessarily fully factorized. This induces a partitioning of the cliques of $P$ into those represented directly by $\bmu$ and those where we define clique marginals as a product distribution of the relevant nodes' marginals in $\bmu$. To account for the energy terms of the full model involving cliques absent in the simple base model, the energy $\dt{{\boldsymbol{\theta}}}{\bmu}$ of the base model is augmented with an extra function of $\bmu$. 
\begin{align}
 L(\mu) =  -\sum_{c \in \mathcal{C}} \dt{{\boldsymbol{\theta}}_c}{\bigotimes_{n \in c} \mu_n} \label{eq:mean-field}
\end{align}
where $\mathcal{C}$ is the set of cliques not included in the tractable sub-model, ${\boldsymbol{\theta}}_c$ are the potentials of the original graphical model corresponding to the missing cliques, and $\bigotimes_n \mu_n$ represents a repeated outer (tensor) product of the node marginals for the nodes in those cliques.

Note $L(\bmu) $ is non-linear and non-convex.  
Our work generalizes~\eqref{eq:mean-field} by allowing arbitrary non-linear interaction terms between components of $\bmu$. This is very powerful -- for example, in our citation extraction experiments in Section \ref{citation-experiments}, expressing these global terms in a standard graphical model would require many factors touching all variables. Local coordinate ascent mean-field can be frustrated by these rigid global terms. Our gradient-based method avoids these issues by updating all marginals simultaneously. 

\textbf{Dual decomposition} is a popular method for performing MAP inference in complex structured prediction models by leveraging repeated calls to MAP in tractable submodels~\citep{komodakis2007mrf,sontag2011introduction}. The family of models solvable with dual decomposition is limited, however, because the terms that link the submodels must be expressible as linear constraints.  Similar MAP techniques~\citep{ravikumar2010message,aguiar2011augmented,fu2013bethe} based on the alternating direction method of multipliers (\textbf{ADMM}) can be adapted for marginal inference, in problems where marginal inference in submodels is tractable. However, the non-local terms are defined as linear functions on settings of graphical model nodes, while our non-linear $L_{\boldsymbol{\psi}}(\bmu)$ terms provide practitioners with an expressive means to learn and enforce regularities of the inference output. 

\textbf{Posterior regularization} (PR) \citep{ganchev2010posterior},  \textbf{learning from measurements} (LFM)~\citet{liang2009learning} , and \textbf{generalized expectations} (GE) \citep{mann2010generalized}, are  a family of closely-related techniques for performing unsupervised or semi-supervised learning of a conditional distribution $P_{\boldsymbol{\theta}}(\by | \bx)$ or a generative model $P_{\boldsymbol{\theta}}(\bx | \by)$ using expectation-maximization (EM), where the E-step for latent variables $\by$ does not come directly from inference in the model, but instead from projection onto a set of expectations obeying global regularity properties. In PR and GE, this yields a projection objective of the form~\eqref{eq:aug-inf2}, where the $L_{\boldsymbol{\psi}}$ terms come from a Lagrangian relaxation of regularity constraints, and ${\boldsymbol{\psi}}$ corresponds to dual variables. Originally, PR employed linear constraints on marginals, but \citet{he2013conll} extend the framework to arbitrary convex differentiable functions. Similarly, in LFM such an inference problem arises because we perform posterior inference assuming that the observations $\by$ have been corrupted under some noise model.~\citet{tarlow2012structured} also present a method for learning with certain forms of non-local losses in a max-margin framework. 

Our goals are very different than the above learning methods. We do not impose non-local terms $L_{\boldsymbol{\psi}}$ in order to regularize our learning process or allow it to cope with minimal annotation. Instead, we use $L_{\boldsymbol{\psi}}$ to increase the expressivity of our model, performing inference for every test example, using a different ${\boldsymbol{\psi}}$, since it depends on input features.  Since we are effectively `learning the regularizer,' on fully-labeled data, our learning approach in Section~\ref{sec:learning} differs from these methods. Finally, unlike these frameworks, we employ non-convex $L_{\boldsymbol{\psi}}$ terms in some of our experiments. The algorithmic consequences of non-convexity are discussed in Section~\ref{sec:non-convex}.

\section{OPTIMIZING THE NON-LOCAL MARGINAL INFERENCE OBJECTIVE}
\label{sec:optimization}

We now present an approach to solving~\eqref{eq:aug-inf2} using non-Euclidean projected gradient methods, which require access to a procedure for marginal inference in the base distribution (which we term the \emph{marginal oracle}), as well as access to the gradient of the energy function $L_{\boldsymbol{\psi}}$. We pose these algorithms in the \emph{composite minimization} framework, which gives us access to a wide variety of algorithms that are discussed in the supplementary material.

\subsection{CONVEX OPTIMIZATION BACKGROUND}

Before presenting our algorithms, we review several definitions from convex analysis \citep{rockafellar1997convex}.

We call a function $\varphi$ \emph{$\sigma$-strongly convex} with respect to a norm $\|\cdot\|_P$, if for all $x,y \in \text{dom}(\varphi)$,
\begin{align*}
\varphi(y) \ge \varphi(x) + \nabla \varphi(x)^T(y - x) + \frac{\sigma}{2} \|y - x\|^2_P.
\end{align*}

\begin{proposition}[e.g. \citet{beck2003mirror}]
The negative entropy function $-H(x)=\sum_i x_i \log x_i$ is 1-strongly convex with respect to the 1-norm $\|\cdot\|_1$ over the interior of the simplex $\Delta$ (restricting $\text{dom}(H)$ to $\textbf{int}(\Delta)$).
\end{proposition}

Given a smooth and strongly convex function $\varphi$, we can also define an associated generalized (asymmetric) distance measure called the \emph{Bregman divergence} \citep{bregman1967relaxation} generated by $\varphi$,
\begin{align*}
B_\varphi(x, x_0) = \varphi(x) - \varphi(x_0) - \dt{\nabla \varphi(x_0)}{x - x_0}.
\end{align*}
For example, the KL divergence is the Bregman divergence associated to the negative entropy function, and the squared Euclidean distance is its own associated divergence. 

\emph{Composite minimization} \citep{Passty1979383} is a family of techniques for minimizing functions of the form $h = f + R$, where we have an oracle that allows us to compute minimizations over $R$ in closed form (usually $R$ here takes the form of a regularizer). Problems of this form are often solved with an algorithm called \emph{proximal gradient}, which minimizes $h(x)$ over some convex set $X$ using:
\begin{align*}
x_{t+1} = \argmin_{x \in X} ~\dt{\nabla f(x_t)}{x} + \frac{1}{2\eta_t}\|x - x_t\|_2^2 + R(x),
\end{align*}
for some decreasing sequence of learning rates $\eta_t$. Note that because of the requirement $x \in X$, proximal gradient generalizes projected gradient descent -- since unconstrained minimization might take us out of the feasible region $X$, computing the update requires projecting onto $X$.

But there is no reason to use the squared Euclidean distance when computing our updates and performing the projection. In fact, the squared term can be replaced by any Bregman divergence. This family of algorithms includes the \emph{mirror descent} and \emph{dual averaging} algorithms \citep{beck2003mirror,nesterov2009primal}.

\begin{algorithm}[tb]
   \caption{Bethe-RDA}
   \label{alg:b-rda}
\begin{algorithmic}
   \STATE {\bfseries Input:} parameters ${\boldsymbol{\theta}}$, energy function $L_{\boldsymbol{\psi}}(\mu)$
   \STATE set ${\boldsymbol{\theta}}_0 = {\boldsymbol{\theta}}$
   \STATE set $\mu_0$ to prox-center $\text{MARGINAL-ORACLE}({\boldsymbol{\theta}}_0)$
   \STATE ${\bar g_0} = 0$
   \REPEAT
   \STATE $\beta_t = \text{constant} \ge 0$
   \STATE ${\bar g_t} = \frac{t-1}{t}{\bar g_{t-1}} + \frac{1}{t}\nabla L(\mu_t)$
   \STATE ${\boldsymbol{\theta}}_t = {\boldsymbol{\theta}} - \frac{t}{t + \beta_t}{\bar g_t}$
   \STATE $\mu_t = \text{MARGINAL-ORACLE}({\boldsymbol{\theta}}_t)$
   \UNTIL{$\text{CONVERGED}(\mu_t, \mu_{t-1})$}
\end{algorithmic}

\end{algorithm}

We base our projected inference algorithms on \emph{regularized dual averaging} (RDA) \citep{xiao2010dual}. The updates are:
\begin{align}
\label{eq:rda}
x_{t+1} = \argmin_{x \in X} ~\dt{{\bar g_t}}{x} + \frac{\beta_t}{t}\varphi(x) + R(x),
\end{align}
where ${\bar g_t} = \frac{1}{t} \sum_k^t\nabla f(x_k)$ is the average gradient of $f$ encountered so far. One benefit of RDA is that it does not require the use of a learning rate parameter ($\beta_t = 0$) when using a strongly convex regularizer. RDA can be interpreted as doing a projection onto $X$ using the Bregman divergence generated by the strongly convex function $\varphi + R$.

\subsection{OUR ALGORITHM}

These non-Euclidean proximal algorithms are especially helpful when we are unable to compute a projection in terms of Euclidean distance, but can do so using a different Bregman divergence. We will show that this is exactly the case for our problem of projected inference: the marginal oracle allows us to project in terms of KL divergence.

However, to maintain tractability we avoid using the entropy function $H$ on the exponentially-large simplex $\Delta$, and instead optimize over the structured, factorized marginal polytope $\mathcal{M}$ and its corresponding structured Bethe entropy $\hb$. For tree-structured models, $H$ and $\hb$ have identical values, but different inputs. It remains to show the strong convexity of $-\hb$ so we can use it in RDA.

\begin{proposition}
For trees with $n$ nodes, the negative Bethe entropy function $-\hb$ is $\frac{1}{2}(2n - 1)^{-2}$-strongly convex with respect to the 2-norm over the interior of the marginal polytope $\mathcal{M}$.
\end{proposition}
\begin{proof}
\renewcommand{\qedsymbol}{}
Consequence of Lemma 1 in \citet{fu2013bethe}.
\vspace{-15pt}
\end{proof}

With these definitions in hand, we present Bethe-RDA projected inference Algorithm \ref{alg:b-rda}. This algorithm corresponds to instantiating \eqref{eq:rda} with $R=-\hb - \dt{{\boldsymbol{\theta}}}{\mu}$ and $\varphi=-\hb$. Note the simplicity of the algorithm when choosing $\beta_t = 0$. It is intuitively appealing that the algorithm amounts to no more than calling our marginal inference oracle with iteratively modified parameters. 

\begin{proposition}
For convex energy functions and convex $-\hb$, the sequence of primal averages of Algorithm \ref{alg:b-rda} converges to the optimum of the variational objective \eqref{eq:aug-inf2} with suboptimality of $O(\frac{\text{ln}(t)}{t})$ at time $t$.
\end{proposition}
\begin{proof}
This follows from Theorem 3 of \cite{xiao2010dual} along with the strong convexity of $-\hb$.
%\vspace{-10pt}
\end{proof}

If we have more structure in the energy functions, specifically a Lipschitz-continuous gradient, we can modify the algorithm to use Nesterov's acceleration technique and achieve a convergence of $O(\frac{1}{t^2})$. Details can be found in Appendix \ref{sec:supp-acc-rda}. Additionally, in practice these problems need not be solved to optimality and give stable results after a few iterations, as demonstrated in Figure \ref{fig:acc-vs-iters}.

\subsection{INFERENCE WITH NON-CONVEX, NON-LOCAL ENERGIES}
\label{sec:non-convex}

An analogy can be made here to loopy belief propagation -- even in the case of non-convex loss functions (and even non-convex entropy functions with associated inexact marginal oracles), the updates of our inference (and learning) algorithms are well-defined. Importantly, since one of our motivations for developing non-local inference was to generalize mean field inference, and the additional penalty terms are non-convex in that case, we would like our algorithms to work for the non-convex case as well.

Unlike loopy belief propagation, however, since we derive our algorithms in the framework of composition minimization, we have access to a wealth of theoretical guarantees. Based on results from the theory of optimization with first-order surrogate loss functions \citep{mairal2013optimization}, in Appendix \ref{sec:supp-non-convex} we propose a small modification to Algorithm \ref{alg:b-rda} with an asymptotic convergence condition even for non-convex energies. In practice we find that the unmodified Algorithm \ref{alg:b-rda} also works well for these problems, and experimentally in Section \ref{ocr-experiments}, we see good performance in both inference and learning with non-convex energy functions.

\section{LEARNING MODELS WITH NON-LOCAL ENERGIES}
\label{sec:learning}

\begin{algorithm}[tb]
   \caption{Learning with non-local energies}
   \label{alg:learning}
\begin{algorithmic}
   \STATE {\bfseries Input:} examples ${\bx_i, \by_i}$ and inference oracle $\text{MARG}()$
\STATE  for distributions with the clique structure of $P_{{\boldsymbol{\theta}}}(\by|\bx)$. 
 \STATE {\bfseries Output:} parameters $({\boldsymbol{\theta}},{\boldsymbol{\psi}})$ for $P_c(\by | \bx)$. 
%\STATE
\REPEAT{
\STATE //E-Step
\FORALL{$(\bx_i,\by_i)$}
\STATE $\bmu_i \leftarrow $ (Algorithm~\ref{alg:b-rda}) // using ${\boldsymbol{\theta}},{\boldsymbol{\psi}}$ and MARG()
\STATE $\rho_i \leftarrow $ (Proposition~\ref{prop:1}) // using ${\boldsymbol{\psi}},\bmu_i$
\STATE // note $Q_i(\by_i)$ is a CRF with potentials ${\boldsymbol{\theta}} + \rho_i$. 
\ENDFOR
%\STATE
\STATE //M-Step (gradient-based learning of CRF parameters)
\REPEAT{
\STATE $m_i \leftarrow \text{MARG}(Q_i) \; \forall j$ //standard CRF inference
\STATE $\nabla_{\boldsymbol{\theta}} \leftarrow \sum_i S(\by_i) - m_i$
\STATE $\nabla_{\boldsymbol{\psi}} \leftarrow \sum_i \frac{d \rho_i}{d {\boldsymbol{\psi}}}^\top \left(S(\by_i) - m_i\right)$
\STATE ${\boldsymbol{\theta}} \leftarrow$ Gradient-Step(${\boldsymbol{\theta}}, \nabla_{\boldsymbol{\theta}}$)
\STATE ${\boldsymbol{\psi}} \leftarrow$ Gradient-Step(${\boldsymbol{\psi}}, \nabla_{\boldsymbol{\psi}}$)
}
\UNTIL{converged}
}
\UNTIL{converged OR iter $>$ max\_iters}

\end{algorithmic}
\end{algorithm}

\begin{algorithm}[tb]
   \caption{Doubly-stochastic learning with $L_{\boldsymbol{\psi}}$ given by a sum of scalar functions of linear measurements~\eqref{eq:ell2}. }
   \label{alg:learning2}
\begin{algorithmic}
   \STATE {\bfseries Input:} examples ${\bx_i, \by_i}$ and $\text{MARGINAL-ORACLE}()$
\STATE  for distributions with the clique structure of $P_{{\boldsymbol{\theta}}}(\by|\bx)$. 
 \STATE {\bfseries Output:} parameters $({\boldsymbol{\theta}},{\boldsymbol{\psi}})$ for $P_c(\by | \bx)$. 
%\STATE
\REPEAT{
\STATE sample $(\bx_i,\by_i)$ randomly
\STATE $\bmu_i \leftarrow $ (Algorithm~\ref{alg:b-rda}) 
\STATE $\nabla_{\boldsymbol{\theta}} \leftarrow  S(\by_i) - \bmu_i$
\STATE $\nabla_{{\boldsymbol{\psi}}_j} \leftarrow  \nabla \ell_j (\bmu_i) a_j^\top \left( S(\by_i) - \bmu_i\right)$
\STATE ${\boldsymbol{\theta}} \leftarrow$ Gradient-Step(${\boldsymbol{\theta}}, \nabla_{\boldsymbol{\theta}}$)
\STATE ${\boldsymbol{\psi}} \leftarrow$ Gradient-Step(${\boldsymbol{\psi}}, \nabla_{\boldsymbol{\psi}}$)

}
\UNTIL{converged OR iter $>$ max\_iters}

\end{algorithmic}
\end{algorithm}

We seek to learn the parameters ${\boldsymbol{\theta}}$ and ${\boldsymbol{\psi}}$ of the underlying CRF base model and $L_{\boldsymbol{\psi}}$, respectively. Let $S = \{\by_i,\bx_i\}$ be $n$ training examples.  Let $Q( \by_i ; \bmu_i)$  be the variational distribution for $\by_i$  resulting from applying Proposition~\ref{prop:var}. Namely, $Q( \by_i ; \bmu_i)$  is an MRF with parameters 
\begin{equation}
\boldsymbol{\rho_i} \defeq \btheta - \nabla_{\bmu } L_{\boldsymbol{\psi}}(\bmu_i). \label{eq:q1main}
\end{equation}
We employ the notation $Q( \by_i ; \bmu_i)$ to highlight the role of $\bmu_i$: for  a given $(\by_i,\bx_i)$ pair, the family of variational distributions over $\by_i$ is indexed by possible values of $\bmu_i$ (recall we suppress the explicit dependence of $\btheta$ and ${\boldsymbol{\psi}}$ on $\bx$). Finally, define the shorthand $M = \{\mu_1, \ldots, \mu_n\}$. 

${\boldsymbol{\psi}}$ interacts with the data in a complex manner that prevents us from using standard learning techniques for the exponential family. Namely, we can not easily differentiate a likelihood with respect to ${\boldsymbol{\psi}}$, since this requires differentiating the output $\bmu$ of a convex optimization procedure, and the extra $L_{\boldsymbol{\psi}}$ term in~\eqref{eq:aug-inf2} prevents the use of conjugate duality relationships available for the exponential family. We could have used automatic methods to differentiate the iterative inference procedure~\citep{stoyanov2011empirical,domke2012generic}, but found our learning algorithm works well.

We employ a variational learning algorithm, presented in Algorithm~\ref{alg:learning}, alternately updating the parameters $M$ of our tractable CRF-structured variational distributions, and updating the parameters $({\boldsymbol{\theta}},{\boldsymbol{\psi}})$ assuming the following surrogate likelihood given by these CRF approximations:
\begin{align}
L({\boldsymbol{\theta}},{\boldsymbol{\psi}};M) &=  \sum_i \log Q(y_i;\bmu_i). \label{eq:surrogate}
\end{align}
Given ${\boldsymbol{\theta}}$ and ${\boldsymbol{\psi}}$, we update $M$  using Algorithm~\ref{alg:b-rda}. Given $M$, we update ${\boldsymbol{\theta}}$ and ${\boldsymbol{\psi}}$ by taking a single step in the direction of the gradient of the surrogate likelihood~\eqref{eq:surrogate}. We avoid taking more than one gradient step, since the gradients for ${\boldsymbol{\theta}}$ and ${\boldsymbol{\psi}}$ depend on $M$ and an update to ${\boldsymbol{\theta}}$ and ${\boldsymbol{\psi}}$ will break the property that $\bmu\left(Q(\by ; \bmu_i)\right) = \bmu_i$. Therefore, we recompute $\bmu_i$ every time we update the parameters. %In practice, we employ a stochastic version of Algorithm~\ref{alg:learning}, where we update ${\boldsymbol{\theta}}$ and ${\boldsymbol{\psi}}$ based on a gradient from a single $(\bx_i,\by_i)$. 

Overall, it remains to show how to compute gradients of~\eqref{eq:surrogate}. For ${\boldsymbol{\theta}}$, we have the standard CRF likelihood gradient~\citep{sutton2006introduction}:
\begin{align}
\nabla_{\boldsymbol{\theta}} L({\boldsymbol{\theta}},{\boldsymbol{\psi}};M)  = \sum_i S(\by_i) - \bmu_i ~\label{eq:crf-ll}. 
\end{align}
For ${\boldsymbol{\psi}}$, we have:
\begin{align}
\nabla_{\boldsymbol{\psi}} L({\boldsymbol{\theta}},{\boldsymbol{\psi}};M)  = \sum_i \frac{d \boldsymbol{\rho_i}}{d {\boldsymbol{\psi}}} \frac{d}{d \boldsymbol{\rho_i}}  \log  Q(\by_i;\bmu_i) . 
\end{align}
From~\eqref{eq:q1main}, $\frac{d}{d \boldsymbol{\rho_i}}  \log  Q(\by_i;\bmu_i) $ is also $S(\by_i) - \bmu_i$and

% Using the variational interpretation from Section \ref{sec:var-and-map}

\begin{align}
\frac{d \boldsymbol{\rho_i}}{d {\boldsymbol{\psi}}} = \frac{d}{d {\boldsymbol{\psi}}} \frac{d}{d \bmu}  L_{\boldsymbol{\psi}}(\bmu) \label{eq:D-L}
\end{align}

Clearly, this depends on the structure of $L_{\boldsymbol{\psi}}$.  Consider the parametrization~\eqref{eq:ell1}. 
With this, we have:
\begin{align}
\frac{\partial}{\partial {\boldsymbol{\psi}}_j} \frac{d}{d \bmu}  L_{\boldsymbol{\psi}}(\bmu) =  \nabla \ell_j(\bmu) \frac{d }{ d \bmu}\ell_j(\bmu)
\end{align}
Therefore, we have $\frac{\partial}{\partial {\boldsymbol{\psi}}_j} \log  Q(\by_i;\bmu_i) =   \nabla \ell_j(\bmu) \frac{d }{ d \bmu} \ell_j(\bmu)^\top \left(S(\by) - \bmu_i \right)$.  For linear measurements~\eqref{eq:ell2}, this amounts to
\begin{align}
\nabla \ell(\bmu)  \left(a_j^\top S(\by) - a_j^\top \bmu_i \right). \label{eq:lm}
\end{align}
This has a simple interpretation: the gradient with respect to ${\boldsymbol{\psi}}_j$ equals the gradient of the scalar loss $\ell_j$ at the current marginals $\bmu_j$ times the difference in linear measurements between the ground truth labels and the inferred marginals. 

Algorithm~\ref{alg:learning} has an expensive double-loop structure. In practice it is sufficient to employ a `doubly-stochastic' version given in Algorithm~\ref{alg:learning2}, where we sample a training example $(\bx_i, \by_i)$ and use this to only perform a single gradient step on ${\boldsymbol{\theta}}$ and ${\boldsymbol{\psi}}$. To demonstrate the simplicity of implementing our learning algorithm, we avoid any abstract derivative notation in Algorithm~\ref{alg:learning2} by specializing it to the case of~\eqref{eq:lm}. In our experiments, however, we sometimes do not use linear measurements. Overall, all our experiments use the fast doubly-stochastic approach of Algorithm~\ref{alg:learning2} solely, since it performs well. In general, our learning algorithms are not guaranteed to converge because we approximate the complex interaction between ${\boldsymbol{\psi}}$ and $\bmu$ with alternating updates. In practice, however, terminating after a fixed number of iterations yields models that generalize well. 

Finally, recall that the notation $L_{\boldsymbol{\psi}}(\bmu_i)$ suppresses the potential dependence of ${\boldsymbol{\psi}}$ on $\bx_i$. We assume each ${\boldsymbol{\psi}}_j$ is a differentiable function of features of $\bx_i$. Therefore, in our experiments where ${\boldsymbol{\psi}}$ depends on $\bx_i$, we perform gradient updates for the parametrization of ${\boldsymbol{\psi}}(\bx)$ via further application of the chain rule. 

\section{EXPERIMENTS}

\subsection{CITATION EXTRACTION}
\label{citation-experiments}

\begin{table}[ht]
\begin{center}
\begin{tabular}{ |c | c | c | c|}
\hline
Model & F1\\
\hline
Our Baseline & 94.47\\
Non-local Energies & 95.47\\
\hline
Baseline \citep{anzaroot2014learning} & 94.41\\
Soft-DD \citep{anzaroot2014learning} & 95.39\\
\hline
\end{tabular}
\caption{Comparison of F1 scores on Citation Extraction dataset. We compare MAP inference F1 scores of our non-local energy model and the specialized dual decomposition model of \citet{anzaroot2014learning}. Both variants learn global regularities that significantly improve performance. \label{tab:citation-results}}
\end{center}
\end{table}

\begin{figure}[ht]
\label{fig:acc-vs-iters}
\centering
\includegraphics[width=0.30\textwidth]{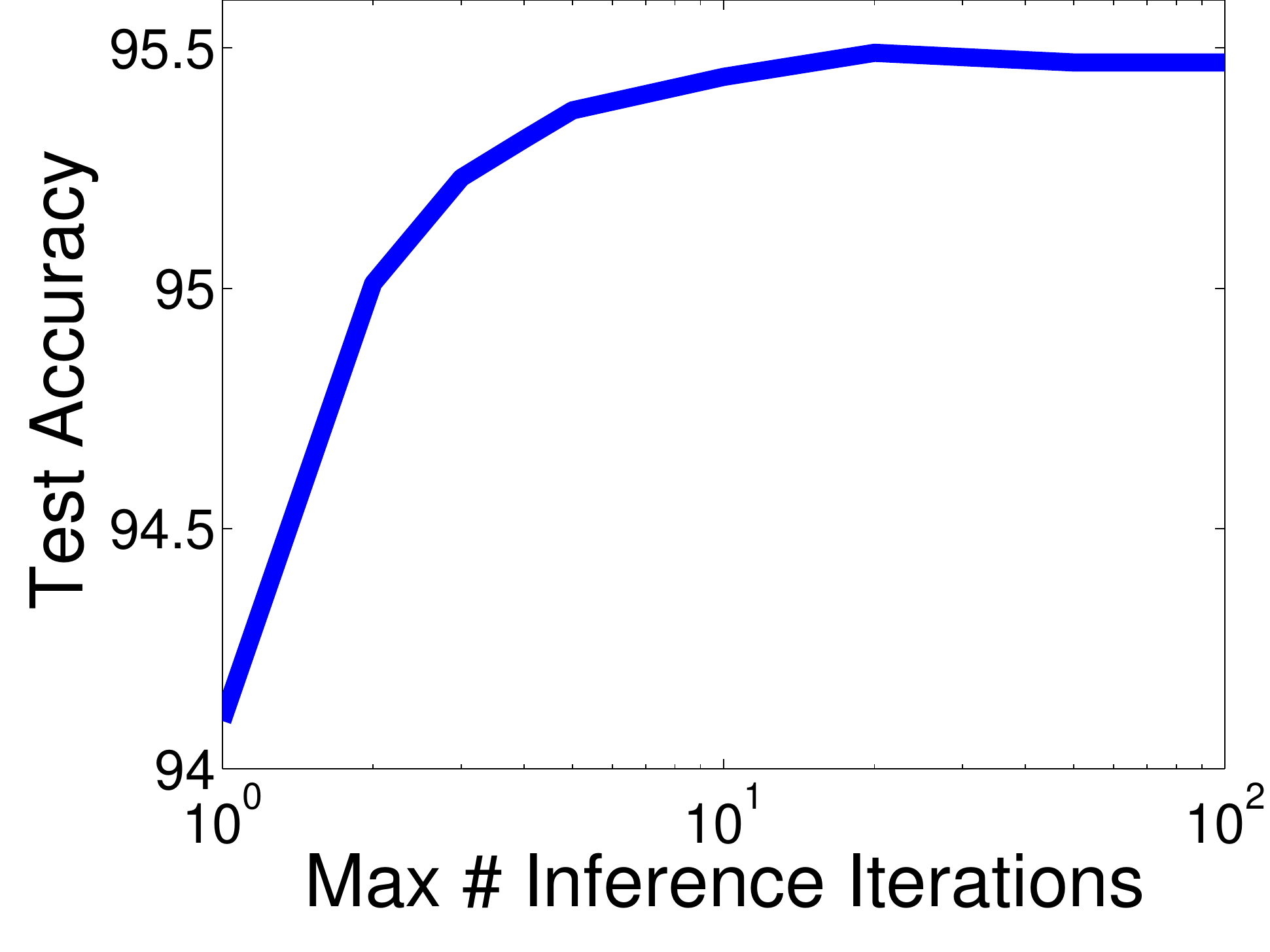}
\caption{Citation extraction F1 when limiting maximum number of test-time inference iterations. Most of our accuracy gain is captured within the first 5-10 iterations.}
\label{fig:cit}
\end{figure}

We first apply our algorithm to the NLP task of performing text field segmentation on the UMass citation dataset \citep{anzaroot2013new}, which contains 
%1456 training, 366 testing, and 659 development 
strings of citations from research papers, segmented into fields (author, title, etc.). Our modeling approach, closely follows~\citet{anzaroot2014learning}, who extract segmentations using a linear-chain segmentation model, to which they add a large set of `soft' linear global regularity constraints. 

 Let $\by$ be a candidate labeling. Imagine, for example, that we constrain predicted segmentations to have no more predicted last names than first names. Then, the numbers of first and last names can be computed by linear measurements $a_{\text{first}}^\top S(\by)$ and $a_{\text{last}}^\top S(\by)$, respectively. A hard constraint on $\by$ would enforce $a_{\text{first}}^\top S(\by) - a_{\text{last}}^\top S(\by) = 0$.  This is relaxed in~\citet{anzaroot2014learning} to a penalty term 
\begin{align}
c \ell_h\left(a_{\text{first}}^\top S(\by) - a_{\text{last}}^\top S(\by)\right)
\end{align}
that is added to the MAP inference objective, where $\ell_h(x) = \max(1-x,0)$ is  a hinge function. For multiple soft constraints, the overall prediction problem is
\begin{align}
\argmin_\by \dt{-{\boldsymbol{\theta}}}{S(\by)} + \sum_j c_j \ell_h\left(a_j^\top S(\by) \right), \label{eq:anz}
\end{align}
where ${\boldsymbol{\theta}}$ are the parameters of the underlying linear-chain model. They use a dual decomposition style algorithm for solving~\eqref{eq:anz}, that crucially relies on the specific structure of the hinge terms $\ell_h$. They learn the $c_j$ for hundreds of `soft constraints' using a perceptron-style algorithm. 

We consider the same set of measurement vectors $a_j$, but impose non-local terms that act on~\textit{marginals} $\bmu$ rather than specific values $\by$. Further, we use \emph{smoothed} hinge functions, which improve the convergence rate of inference~\citep{rennie2005smooth}. We find the variational distribution by solving the marginal inference version of~\eqref{eq:anz}, an instance of our inference framework with linear measurements \eqref{eq:ell2}:
\begin{align}
\argmin_{\bmu} \dt{-{\boldsymbol{\theta}}}{\bmu}- \hb(\bmu) +  \sum_j c_j \ell_h\left(a_j^\top \bmu \right), \label{eq:us-citation}
\end{align}

As in ~\citet{anzaroot2014learning}, we first learn chain CRF parameters ${\boldsymbol{\theta}}$ on the training set. Then, we learn the $c_j$ parameters on the development set, using Algorithm~\ref{alg:learning2}, and tune hyperparameters for development set performance. At both train and test time, we ignore any terms in~\eqref{eq:us-citation} for which $c_j <0$. 

We present our results in Table \ref{tab:citation-results}, measuring segment-level F1. We can see that our baseline chain has slightly higher accuracy than the baseline approach of \citet{anzaroot2014learning}, possibly due to optimization differences. Our augmented model (Non-Local Energies) matches and very slightly beats their soft dual decomposition (Soft-DD) procedure. This is especially impressive because they employ a specialized linear-programming solver and learning algorithm adapted to the task of MAP inference under hinge-loss soft constraints, whereas we simply plug in our general learning and inference algorithms for non-local structured prediction -- applicable to any set of energy functions. 

Our comparable performance provides experimental evidence for our intuition that preferences about MAP configurations can be expressed (and ``relaxed'') as functions of expectations. \citet{anzaroot2014learning} solve a penalized MAP problem directly, while our prediction algorithm first finds a distribution satisfying these preferences, and then performs standard MAP inference in that distribution. 

Finally, in Figure~\ref{fig:cit} we present results demonstrating that our algorithm's high performance can be obtained using only 5-10 calls per test example to inference in the underlying chain model. In Section~\ref{sec:supp-cit}, we analyze the empirical convergence behavior of Algorithm~\ref{alg:b-rda}.

\subsection{HANDWRITING RECOGNITION}
\label{ocr-experiments}

\begin{table}[ht]
\begin{center}
\begin{tabular}{ | c | c | c | c | c | c | }
\hline
N-Grams& 2 & 3 & 4 & 5 & 6\\
\hline
Accuracy& 85.02 & 96.20 & 97.21 & 98.27 & 98.54 \\
\hline
\end{tabular}
\caption{Character-wise accuracy of Structured Prediction Cascades \citep{2012arXiv1208.3279W} on OCR dataset.\label{tab:ocr-spc-results}}
\end{center}
\end{table}

\begin{table}[ht]

\begin{center}
\begin{tabular}{ |c | c|}
\hline
SPC \citep{2012arXiv1208.3279W}  & Accuracy\\
\hline
2-gram & 85.02\\
3-gram& 96.20\\
4-gram & 97.21\\
5-gram & 98.27\\
6-gram & 98.54\\
\hline
\end{tabular}
\caption{Character-wise accuracy of our baselines, and models using learned non-local energies on Handwriting Recognition dataset. Note that word classifier baseline is also given in character-wise accuracy for comparison. \label{tab:ocr-global-results}}
\end{center}
\end{table}

\begin{table}[ht]

\begin{center}
\begin{tabular}{ |c | c|}
\hline
Model & Accuracy\\
\hline
2-gram (base model) & 84.93\\
$L_{\boldsymbol{\psi}}^u$ & 94.01\\
$L_{\boldsymbol{\psi}}^u$ (MM) & 94.96\\
$L_{\boldsymbol{\psi}}^w$ & 98.26\\
$L_{\boldsymbol{\psi}}^w$ (MM) & \bf98.83\\
\hline
55-Class Classifier (MM) & 86.06\\
\hline
\end{tabular}
\caption{Character-wise accuracy of our baselines, and models using learned non-local energies on Handwriting Recognition dataset. Note that word classifier baseline is also given in character-wise accuracy for comparison. \label{tab:ocr-global-results}}
\end{center}
\end{table}

We next apply our algorithms to the widely-used handwriting recognition dataset of \citet{koller2004max}. %, containing 6877 handwritten words with an average length of 8 characters per word.  Each character is represented as a 16 by 8 binary image. 
We follow the setup of \citet{2012arXiv1208.3279W}, splitting the data into 10 equally sized folds, using 9  for training and one to test. We report the cross-validation results across all 10 folds.

The \emph{structured prediction cascades} of \citet{2012arXiv1208.3279W} achieve high performance on this dataset by using extremely high order cliques of characters (up to 6-grams), for which they consider only a small number of candidate outputs. Their state-of-the-art results are reproduced in Table \ref{tab:ocr-spc-results}. The excellent performance of these large-clique models is consequence of the fact that the data contains only 55 unique words, written by 150 different people. Once the model has access to enough higher-order context, the problem becomes much easier to solve.

With this in mind, we design two non-convex, non-local energy functions. These energies are intended to regularize our predictions to lie close to known elements of the vocabulary. Our base model is a standard linear-chain CRF with image features on the nodes, and no features on the bigram edge potentials. Let $U(\mu) = \sum_n \mu_n $ be a function that takes the concatenated vector of node and edge marginals and sums up all of the node marginals, giving the global unigram expected sufficient statistics. Let $\{u_i\} = \{U(\mu(y_i))\}$ indicate the set of all such unique vectors when applying $U$ to the train set empirical sufficient statistics for each data case $y_i$. Simply, this gives 55 vectors $u_i$ of length 26 containing the unigram counts for each unique word in the train set.

Our intuition is that we would like to be able to ``nudge'' the results of inference in our chain model by pulling the inferred $U(\mu)$ to be close to one of these global statistics vectors. We add the following non-convex non-local energy function to the model:
\begin{align}
L^u_{\boldsymbol{\psi}}(\mu) = {\boldsymbol{\psi}} \min_{i} \|u_i - U(\mu)\|_1 .
\end{align}
We learn two variants of this model, which differently parametrize the dependence of ${\boldsymbol{\psi}}$ on $\bx$. The first has a single bias feature on the non-local energy. The second conditions on a global representation of the sequence: concretely, we approximate the RBF \emph{kernel mean map} (MM) \citep{smola2007hilbert} using random Fourier features (RFF) \citep{rahimi2007random}. This simply involves multiplying each image feature vector in the sequence by a random matrix with $\sim 1000$ rows, applying a pointwise non-linearity, and taking ${\boldsymbol{\psi}}$ to be a linear function of the average vector.

Results of these experiments can be seen in Table \ref{tab:ocr-global-results}. Adding the non-local energy brings our performance well above the baseline bigram chain model, and our training procedure is able to give substantially better performance when ${\boldsymbol{\psi}}$ depends on the above input features.

The energy $L^u_{\boldsymbol{\psi}}$, based on unigram sufficient statistics, is not able to capture the relative ordering of letters in the vocabulary words, which the structured prediction cascades models do capture. This motivates us to consider another energy function. Let $\{w_i\}=\{\mu_n(y_i)\}$ be the set of unique vectors of concatenated node marginal statistics for the train set. This gives 55 vectors of length $l_i * 26$, where $l_i$ is the length of the $i$th distinct train word. Next, we define a different energy function to add to our base chain model:
\begin{align}
L^w_{\boldsymbol{\psi}}(\mu) = {\boldsymbol{\psi}}\min_{i} \|w_i - \mu\|_1 .
\end{align}
Once again we implement featurized and non-featurized versions of this model. As noted in structured prediction cascades, giving the model access to this level of high-order structure in the data makes the inference problem extremely easy. Our model outperforms the best structured prediction cascades results, and we note again an improvement from using the featurized over the non-featurized ${\boldsymbol{\psi}}$.

Of course, since the dataset has only 55 actual labels, and some of those are not valid for different input sequences due to length mismatches, this is arguably a classification problem as much as a structured prediction problem. To address this, we create another baseline, which is a constrained 55-class logistic regression classifier (constrained to only allow choosing output classes with appropriate lengths given the input). We use our same global mean-map features from the $L_{\boldsymbol{\psi}}^*~(MM)$ variants of the structured model and report these results in Table \ref{tab:ocr-global-results}. We also tune the number of random Fourier features as a hyperparameter to give the classifier as much expressive power as possible. As we can see, the performance is still significantly below the best structured models, indicating that the interplay between local and global structure is important.

\subsection{COLLECTIVE GRAPHICAL MODELS}
\vspace{-5pt}

\begin{table}
\begin{center}
\begin{tabular}{ |c | c | c | c|}
\hline
$s$ & 625 & 10k & 50k\\
\hline
Our Method & 0.19 & 2.7 & 14\\
\hline
IP & 2.8 & 93 & 690\\
\hline
\end{tabular}
\caption{Comparison of runtime (in seconds, averaged over 10 trials) between the interior point solver (IP) of~\citet{sheldon2013approximate} v.s. Algorithm~\ref{alg:b-rda} on different CGM problem sizes $s$, the cardinality of the edge potentials in the underlying graphical model, where marginal inference is $O(s)$. 
\label{tab:cgm}
\vspace{-8pt}
}
\end{center}
\end{table}

Next, we demonstrate that that our proximal gradient-based inference framework dramatically speeds up approximate inference in \emph{collective graphical models} (CGMs) \citep{sheldon2011collective}. CGMs are a method for structured learning and inference with noisy aggregate observation data. The large-scale dependency structure is represented via a graphical model, but the nodes represent not just single variables, but aggregate sufficient statistics of large sets of underlying variables, corrupted by some noise model. In previous work, CGMs have been successfully applied to modeling bird migration. Here, the base model is a linear chain representing a time series of bird locations. Each observed variable corresponds to counts from bird watchers in different locations. These observations are assumed to be Poisson distributed with rate proportional to the true count of birds present. The CGM MAP task is to infer the underlying migration patterns.

\citet{sheldon2013approximate} demonstrate that MAP in CGMs is NP-hard, \emph{even for trees}, but that approximate MAP can be performed by solving a problem of the form~\eqref{eq:aug-inf2}:
\vspace{-5pt}
\begin{align}
\mu^* = \argmax_{\bmu} \dt{\btheta}{\bmu} + \hb(\bmu) + \sum_i^n P_i(\mu_i | {\boldsymbol{\psi}} \by_i ) \label{eq:cgm-inf}
\end{align}
where $P_i$ are (concave) Poisson log-likelihoods and each $\by_i$ is an observed bird count.

For the case where the underlying CGM graph is a tree, the `hard EM' learning algorithm of~\citet{sheldon2013approximate} is the same as Algorithm~\ref{alg:learning} specialized to their model. Therefore,~\citet{sheldon2013approximate} provide additional experimental evidence that our alternating surrogate-likelihood optimization works well in practice.

The learning procedure of~\citet{sheldon2013approximate} is very computationally expensive because they solve instances of~\eqref{eq:cgm-inf} using an interior-point solver in the inner loop. For the special case of trees, Algorithm ~\ref{alg:b-rda} is directly applicable to~\eqref{eq:cgm-inf}. Using synthetic data and code obtained from the authors, we compare their generic solver to Algorithm ~\ref{alg:b-rda} for solving instances of~\eqref{eq:cgm-inf}. In Table \ref{tab:cgm}, we see that our method achieves a large speed-up with no loss in solution accuracy (since it solves the same convex problem).
\section{DISCUSSION AND FUTURE WORK}

Our results show that our inference and learning framework allows for tractable modeling of non-local dependency structures, resistant to traditional probabilistic formulations. By approaching structured modeling not via independence assumptions, but as arbitrary penalty functions on the marginal vectors $\bmu$, we open many new modeling possibilities.  Additionally, our generic gradient-based inference method can achieve substantial speedups on pre-existing problems of interest. In future work, we will apply our framework to new problems and new domains. 

\subsubsection*{ACKNOWLEDGEMENTS}

This work was supported in part by the Center for Intelligent Information Retrieval, in part by DARPA under agreement number FA8750-13-2-0020, and in part by NSF grant \#CNS-0958392. The U.S. Government is authorized to reproduce and distribute reprints for Governmental purposes notwithstanding any copyright notation thereon. Any opinions, findings and conclusions or recommendations expressed in this material are those of the authors and do not necessarily reflect those of the sponsor.

\bibliographystyle{icml2015}

\bibliography{paper}

\newpage
\onecolumn
\appendix
\begin{center}
\Large{Supplementary Material}
\end{center}

\section{Variational Approximation}

\label{sec:var-inference}
During learning, reasoning about $P_c(\by | \bx)$  in~\eqref{eq:pc} is difficult, due to the intractability of $Z_{{\boldsymbol{\theta}},{\boldsymbol{\psi}}}$. In response, we approximate it with a variational distribution:
\begin{align}
Q(\by) = \argmin_{Q^\prime} F(Q^\prime; \bx, {\boldsymbol{\theta}}, {\boldsymbol{\psi}}),
\end{align}
where 
\begin{align}
F(Q^\prime) & = KL(Q^\prime(\by) || P_c(\by|\bx)) \nonumber\\
& = -H(Q^\prime) -  \Ex_{Q^\prime} [ \dt{{\boldsymbol{\theta}}}{S(\by)} ] + \Ex_{Q^\prime} [  L_{\boldsymbol{\psi}}(S(\by))]  \label{eq:aug-inf-0}\\
& \approx  -H(Q^\prime) - \dt{{\boldsymbol{\theta}}}{\bmu(Q^\prime)} +  L_{\boldsymbol{\psi}}(\bmu(Q^\prime)).  \label{eq:aug-inf}
\end{align}

Given $\bx$, ${\boldsymbol{\theta}}$, and ${\boldsymbol{\psi}}$, we select $Q$ by minimizing the approximation~\eqref{eq:aug-inf}. Note that the surrogate we minimize is a \textit{lower} bound to~\eqref{eq:aug-inf-0}, as  $\Ex_{Q^\prime} [  L_{\boldsymbol{\psi}}(S(\by))] \geq L_{\boldsymbol{\psi}}(\bmu(Q^\prime))$, by Jensen's inequality and the convexity of $L$. This differs from many mean-field variational inference approaches that minimize an upper bound. 

So far, we have not assumed any structure on $Q$. Next, we show that the minimizer of~\eqref{eq:aug-inf} is  a MRF with the same clique structure as $P_{\boldsymbol{\theta}}$. This provides an alternative derivation of the techniques in Section~\ref{sec:var-and-map}. 

Let $q_\by$ denote the probability under $Q$ of a given joint configuration $\by$. There are exponentially-many such $q_\by$, and $H(Q)$ is the entropy on the simplex $-\sum_\by q_\by \log(q_\by)$. Since $Q$ minimizes~\eqref{eq:aug-inf}, we have the following stationarity condition for every $q_\by$:
\begin{align}
\frac{d}{d q_\by}  \left[- H(Q_{\phi}) - q_\by \log(P_{\btheta}(y|\bx))  +   L_{\boldsymbol{\psi}}(\bmu(Q_{\phi}))\right] + \lambda &= 0
\end{align}
Here, $\lambda$ is a dual variable for the constraint $\sum_\by q_\by = 1$.  Rearranging, we have:
\begin{align}
&Q(\by) = \\
&(1/Z)P_{\btheta}(y|\bx) \exp \left(- \left( \frac{d}{d\bmu}  L_{\boldsymbol{\psi}}(\bmu(Q))\right)^\top\left(\frac{d}{d q_\by} \bmu(Q)\right)\right),\label{eq:q1} 
\end{align}
where $Z$ is a normalizing constant. 

\begin{proposition}
\label{prop:1}
There exists a vector $\rho$ such that  the quantity $\left( \frac{d}{d \bmu}  L_{\boldsymbol{\psi}}(\bmu(Q))\right)^\top\left(\frac{d}{d q_\by} \bmu(Q)\right) = \rho^\top S(\by)$ for all $q_\by$. Furthermore, $\rho$ is a simple, closed-form function of $\bmu(Q)$. 
\end{proposition}
\begin{proof}
We have $\frac{d}{d q_\by} \bmu(Q) = S(\by)$, since $\bmu(Q) = \sum_\by q_\by S(\by)$. Therefore, $\rho = \frac{d}{d \bmu}  L_{\boldsymbol{\psi}}(\bmu(Q))$. 
\end{proof}

\begin{corollary}
\label{corr:1}
Since $P_{\boldsymbol{\theta}}(\by|\bx) \propto \dt{{\boldsymbol{\theta}}}{ S(\by)}$,  Proposition~\ref{prop:1} implies $Q(\by)$ is an MRF with the same clique decomposition as $P_{\boldsymbol{\theta}}(\by | \bx)$.
\end{corollary}

So far,  $Q$ is implicitly defined in terms of its own marginals $\bmu(Q)$. Since we assume $P_{\boldsymbol{\theta}}$ and $P_{\boldsymbol{\psi}}$ have the same sufficient statistics $S(\by)$, we can use the Bethe entropy representation $H(Q) = \hb\left(\mu(Q)\right)$. This transforms~\eqref{eq:aug-inf} to the augmented inference problem~\eqref{eq:aug-inf2}. Therefore, we can directly solve for $\bmu(Q)$, which can then be used to provide a closed-form expression for the CRF distribution $Q$.

\section{Additional Experiments}
\label{sec:supp-cit}
\begin{figure}[ht]

\includegraphics[width=0.45\textwidth]{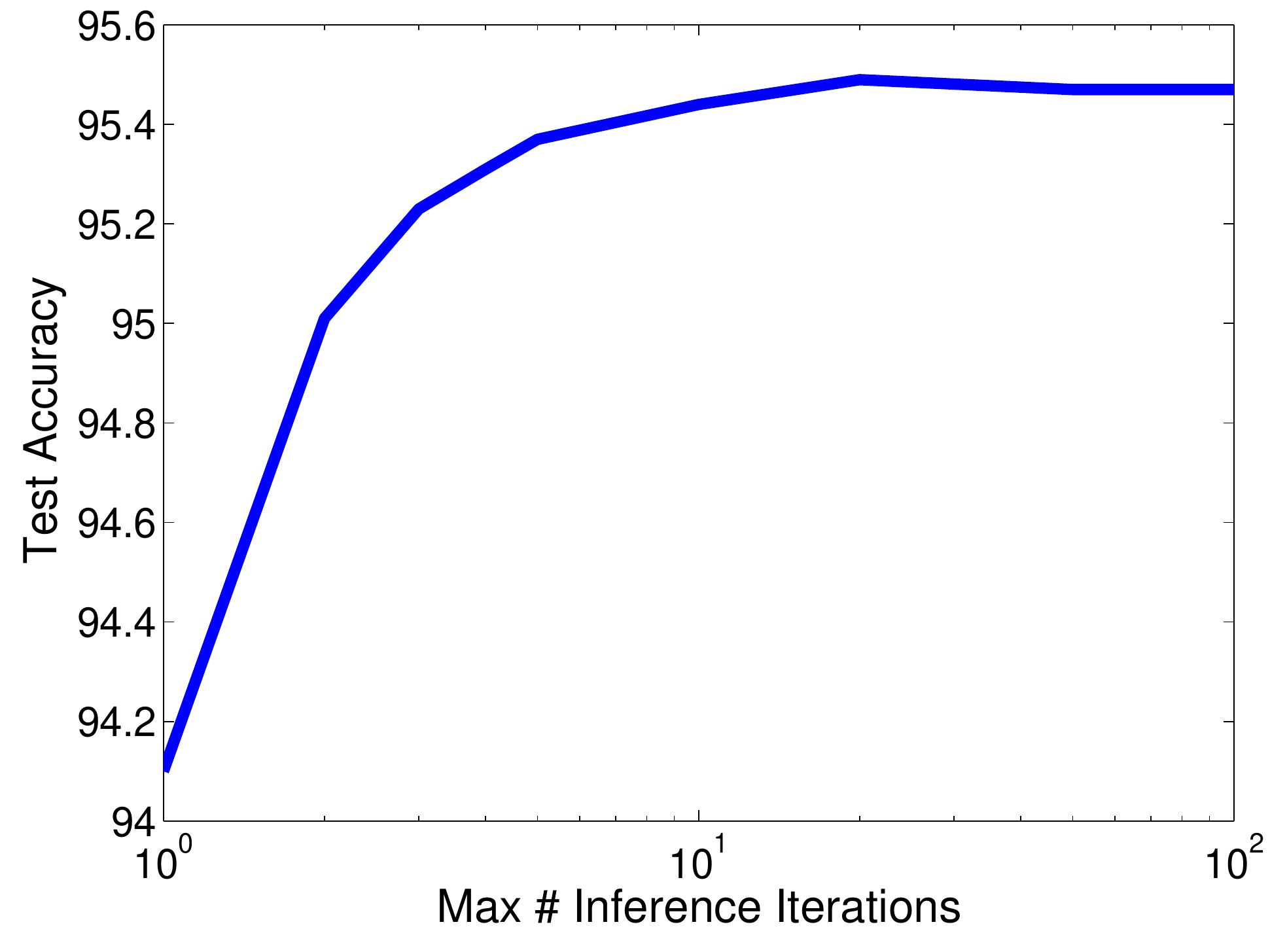}
\caption{The number of iterations taken for inference to converge on test set citations, as a percentage of the total number of test cases. Number of iterations is capped at 40. We can see that the distribution is long tailed. Inference converges within 40 iterations for $93.7$ of examples, and each example takes an average of $9.8$ iterations to converge.\label{fig:iters-until-convergence}}
\end{figure}

In Figure \ref{fig:iters-until-convergence}, we examine the convergence behavior of our algorithm on the citation dataset. This demonstrates that our inference procedure converges quite quickly except for a small number of difficult cases, where the global energy and the local evidence are in significant disagreement.

\section{Non-Convex Energies and Composite Mirror Descent}
\label{sec:supp-non-convex}

We introduce a small modification of Algorithm \ref{alg:b-rda}, along with a rough proof sketch of its convergence even in the case of non-convex energy functions. Because it leans heavily on significant prior work in optimization, it is hard to give a self-contained proof of the results in this section, and our argument takes the form of a proof sketch that appeals to these other works. However, the basic argument simply combines the strong convexity of $\hb$ and its associated Bregman divergence, along with the results of \citet{mairal2013optimization} for the case of composite minimization of non-convex functions using the Euclidean Bregman divergence, and the fact that the local updates performed using entropy $\hb$ as a distance-generating function have a log-barrier function for the constraint set $\mathcal{M}$, effectively bounding the norm of the gradient of $\hb$ when restricted to the set of iterates actually visited during optimization.

While Algorithm \ref{alg:b-rda} was built on the framework of regularized dual averaging (RDA), we introduce a slightly different formulation based on \emph{composite mirror descent} (COMID) \citep{duchi2010composite}. Like RDA, COMID is a gradient method for minimizing functions of the form $h = f + R$. At each time step $t$, COMID makes the update
\begin{align}
\label{comid}
w_{t+1} = \argmin_w \dt{\nabla f (w_t)}{w} + \frac{1}{\eta_t}B_\varphi(w, w_t) + R(w)
\end{align}
where $\varphi$ is some strongly convex function and $B_\varphi$ is its associated Bregman divergence. In Algorithm \ref{alg:b-md}, we present an instantiation of composite mirror descent for our inference problem. 

At first glance, this seems significantly different from our original Algorithm \ref{alg:b-rda}, but remembering that $\nabla \hb (\mu_t) = {\boldsymbol{\theta}}_t$ because of conjugate duality of the exponential family, we can see that it actually only a corresponds to a slight re-weighting of the iterates of Algorithm \ref{alg:b-rda}.

\begin{algorithm}[tb]
  \caption{Bethe-MD}
  \label{alg:b-md}
\begin{algorithmic}
   \STATE {\bfseries Input:} parameters ${\boldsymbol{\theta}}$, energy function $L(\mu)$, learning rate sequence $\{\eta_t\}$
   \STATE set $\mu_0$ to prox-center $\text{MARGINAL-ORACLE}({\boldsymbol{\theta}})$
   \REPEAT
   \STATE $g_t = \nabla \hb(\mu_{t-1}) + \eta_t \nabla L(\mu_{t-1})$
   \STATE $\mu_t = \text{MARGINAL-ORACLE}(\frac{1}{1+\eta_t} (\eta_t {\boldsymbol{\theta}} - g_t))$
   \UNTIL{$\text{CONVERGED}(\mu_t, \mu_{t-1})$}
\end{algorithmic}
\end{algorithm}

First, we give Algorithm \ref{alg:b-md} similar guarantees in the convex setting as we did for Algorithm \ref{alg:b-rda}.

\begin{proposition}
For convex energy functions and convex $-\hb$, given the learning rate sequence $\eta_t = \frac{1}{\lambda t}$, where $\lambda$ is the strong convexity of $-\hb$, the sequence of primal averages of Algorithm \ref{alg:b-md} converges to the optimum of the variational objective \eqref{eq:aug-inf2} with suboptimality of $O(\frac{\text{ln}(t)}{t})$ at time $t$.
\end{proposition}
\begin{proof}
This follows from a standard online-to-batch conversion, along with the strong convexity of $\hb$ and Theorem 7 of \citet{duchi2010composite}.
\end{proof}

Now, having introduced composite mirror descent in ~\eqref{comid}, will lean heavily on the framework for optimization with first-order surrogate losses of \citet{mairal2013optimization} to show that these types of algorithms should converge even in the non-convex case. We now recall a few definitions from that work. 

First, we define the \emph{asymptotic stationary point} condition, which gives us a notion of convergence in the non-convex optimization case.

\begin{definition}[Asymptotic Stationary Point \citep{mairal2013optimization}]
For a sequence $\{{\boldsymbol{\theta}}_n\}_{n \ge 0}$, and differentiable function $f$, we say it satisfies an asymptotic stationary point condition if
\begin{align*}
\lim_{n \to + \infty} \|\nabla f({\boldsymbol{\theta}}_n)\|_2 = 0
\end{align*}
\end{definition}

We call a function $L$-strongly smooth if $L$ is a bound on the largest eigenvalue of the Hessian -- this tells us how the norm of the gradient changes. This is also known as a $L$-Lipschitz continuous gradient. Now we recall the notion of a \emph{majorant first-order surrogate function}.

\begin{definition}[Majorant First-Order Surrogate \citep{mairal2013optimization}]
A function g: $\mathbb{R}^p \to \mathbb{R}$ is a majorant first-order surrogate of $f$ near $\kappa$
when the following conditions are satisfied
\begin{itemize}
\item Majorant: we have $g \ge f$.
\item Smoothness: the approximation error $h = g - f$ is differentiable, and its gradient is $L$-Lipschitz continuous, moreover, we have $h(\kappa)=0$ and $\nabla h(\kappa) = 0$
\end{itemize}
We denote by $\mathcal{S}_L(f,\kappa)$  the set of such surrogates.
\end{definition}

Now we recall the majorant first-order surrogate property for the composite minimization step in the case of Euclidean Bregman divergence (Euclidean distance).

\begin{proposition}[Proximal Gradient Surrogates \citep{mairal2013optimization}]
\label{prop-prox-surrogate}
Assume that $h=f+R$ where $f$ is differentiable with an $L$-Lipschitz gradient. Then, $h$ admits the following majorant surrogate in $\mathcal{S}_{2L}(f,\kappa)$:
\begin{align}
\label{prox-surrogate}
g({\boldsymbol{\theta}}) = f(\kappa) + \nabla f(\kappa)^\top ({\boldsymbol{\theta}} - \kappa) + \frac{L}{2} \|{\boldsymbol{\theta}} - \kappa\|_2^2 + R({\boldsymbol{\theta}})
\end{align}
\end{proposition}

We can use this result to establish a majorant property for the composite mirror descent surrogate \eqref{comid} given a strongly convex and strongly smooth Bregman divergence.

\begin{proposition}[Composite Mirror Descent Surrogates]
\label{prop-comid-surrogate}
Assume that $h=f+R$ where $f$ is differentiable with an $L$-Lipschitz gradient, $\varphi$ is a $\sigma$-strongly convex and $\gamma$-strongly smooth function, and $B_\varphi$ is its Bregman divergence. Then, $h$ admits the following majorant surrogate in $\mathcal{S}_{L + L\frac{\gamma}{\sigma}}(f,\kappa)$:
\begin{align}
\label{comid-surrogate}
g({\boldsymbol{\theta}}) = f(\kappa) + \nabla f(\kappa)^\top ({\boldsymbol{\theta}} - \kappa) + \frac{L}{2\sigma} B_\varphi({\boldsymbol{\theta}}, \kappa) + R({\boldsymbol{\theta}})
\end{align}
\end{proposition}
\begin{proof}
By the definition of strong convexity and the Bregman divergence,~\eqref{comid-surrogate} upper bounds \eqref{prox-surrogate}, so it is a majorant of $h$. Additionally, by the additive property of strong smoothness, we get the strong smoothness constant for the surrogate.
\end{proof}

However, small technical conditions keep Proposition \ref{prop-comid-surrogate} from applying directly to our case. The Bethe entropy $\hb$, and thus its associated Bregman divergence, is not strongly smooth -- its gradient norm is unbounded as we approach the corners of the marginal polytope. However, it is \emph{locally Lipschitz} -- every point in the domain has a neighborhood for which the function is Lipschitz. In practice, since the $-\hb$ mirror descent updates have a barrier function for the constraint set $\mathcal{M}$, our iterative algorithm will never get too close to the boundary of the polytope and it is effectively strongly smooth for purposes of our minimization algorithm. This is not a rigorous argument, but is both intuitively plausible and born out in experiments.

\begin{proposition}
The sequence of iterates $w_t$ from Algorithm \ref{alg:b-md}, when bounded away from the corners of the marginal polytope constraint set $\mathcal{M}$, and for appropriate choice of learning rates $\{\eta_t\}$, convex $-\hb$, and $L$-strongly smooth (but possibly non-convex) energy function $L_{\boldsymbol{\psi}}$, satisfies an asymptotic stationary point condition.
\end{proposition}
\begin{proof}
This follows from application of Proposition \ref{prop-comid-surrogate}, and noting that Algorithm \ref{alg:b-md} corresponds to the generalized surrogate-minimization scheme in Algorithm 1 of \citet{mairal2013optimization}. The asymptotic stationary point condition then follows from Proposition 2.1 of \citet{mairal2013optimization}.  The appropriate learning rates $\{\eta_t\}$ must be chosen by the Lipschitz constant of the gradient of $L_{\boldsymbol{\psi}}$, as well as the effective Lipschitz constant of the gradient of $\hb$, given how far we are bounded from the edge of the constraint set (this effective smoothness constant is determined by the norm of our parameter vector ${\boldsymbol{\theta}}$).
\end{proof}

In this section we have given a rough proof sketch for the asymptotic convergence of our inference algorithms even in the case of non-convex energies. Our heuristic argument for the effective smoothness of the entropy $\hb$ is the most pressing avenue for future work, but we believe it could be made rigorous by examining the norm of the parameter vector and how it contributes to the ``sharpness'' of the barrier function for the mirror descent iterates.

\section{Accelerated Bethe-RDA}
\label{sec:supp-acc-rda}

\begin{algorithm}[tb]
  \caption{Accelerated Bethe-RDA}
  \label{alg:b-acc-rda}
\begin{algorithmic}
   \STATE {\bfseries Input:} parameters ${\boldsymbol{\theta}}$, energy function $L(\mu)$
   \STATE set $\mu_0$ to prox-center $\text{MARGINAL-ORACLE}({\boldsymbol{\theta}})$
   \STATE set $\nu_0 = \mu_0$
   \STATE ${\bar g_0} = 0$
   \REPEAT
   \STATE $c_t = \frac{2}{t+1}$
   \STATE $u_{t} = (1-c_{t})\mu_{t-1} + c_{t} \nu_{t-1}$
   \STATE ${\bar g_{t}} = (1-c_{t}){\bar g_{t-1}} + c_{t} \nabla L (u_{t})$
   \STATE $\nu_t = \text{MARGINAL-ORACLE}(\frac{t(t+1)}{4L + t(t+1)} ({\boldsymbol{\theta}} - {\bar g_{t}}))$
   \STATE $\mu_{t} = (1 - c_t) \mu_{t-1} + c_t \nu_{t}$
   \UNTIL{$\text{CONVERGED}(\mu_t, \mu_{t-1})$}
\end{algorithmic}
\end{algorithm}

If we have $L$-strongly smooth losses ($L$ is a bound on the largest eigenvalue of the Hessian), we can use an accelerated dual averaging procedure to obtain an even faster convergence rate of $O(\frac{1}{t^2})$. Let $D$ be the diameter of the marginal polytope as measured by the strongly convex distance-generating function $\hb$ (using its associated Bregman divergence.) Then Algorithm \ref{alg:b-acc-rda} gives us a convergence rate of $4LD^2/t^2$ by Corollary 7 of \citet{xiao2010dual}.

\end{document}